\documentclass{article}

\usepackage[preprint]{neurips_2026}

\usepackage[utf8]{inputenc}
\usepackage[T1]{fontenc}
\usepackage{hyperref}
\hypersetup{hypertexnames=false}
\usepackage{url}
\usepackage{microtype}
\usepackage{graphicx}
\usepackage{booktabs}
\usepackage{array}
\usepackage{nicefrac}
\usepackage{xcolor}
\usepackage{tikz}
\usetikzlibrary{positioning}

\usepackage{algorithm}
\usepackage{algorithmic}
\usepackage{amsmath}
\usepackage{amssymb}
\usepackage{amsfonts}
\usepackage{mathtools}
\usepackage{amsthm}

\graphicspath{{figures/}}

\theoremstyle{plain}
\newtheorem{theorem}{Theorem}[section]
\newtheorem{proposition}[theorem]{Proposition}

\newtheorem{corollary}[theorem]{Corollary}
\theoremstyle{definition}

\theoremstyle{remark}

\newcommand{\softmax}{\operatorname{softmax}}
\newcommand{\R}{\mathbb{R}}
\newcommand{\calX}{\mathcal{X}}
\newcommand{\calB}{\mathcal{B}}
\newcommand{\calT}{\mathcal{T}}
\newcolumntype{L}[1]{>{\raggedright\arraybackslash}p{#1}}

\title{Vertex-Softmax: Tight Transformer Verification via Exact Softmax Optimization\thanks{Code: \url{https://github.com/navidrezazad/VertexSoftmax}.}}

\author{%
  Navid Rezazadeh\thanks{University of California, Irvine, \texttt{nrezazad@uci.edu}. \quad \textsuperscript{\ddag} Carnegie Mellon University, \texttt{agholami@andrew.cmu.edu}.}\\
  \And
  Arash Gholami Davoodi\footnotemark[3]\\
}

\begin{document}

\maketitle

\begin{abstract}
Certified verification of transformer attention requires bounding the softmax function over interval constraints on the pre-softmax scores. Existing verifiers relax softmax independently of the downstream objective, leaving avoidable slack. We prove that the exact optimum of this score-box problem is attained at a vertex of the constraint box, and establish a threshold structure theorem showing that, after sorting the objective coefficients, the optimum lies among only linearly many candidates, yielding the Vertex-Softmax primitive with log-linear complexity in the sequence length. We further prove a formal optimality result showing that Vertex-Softmax is the tightest sound bound obtainable from score intervals alone, characterizing precisely what additional structure (score correlations, score-value coupling) is needed for further improvement. Integrated into a CROWN (Convex Relaxation based Optimization for Worst-case Neurons)-style verifier with a formal soundness guarantee, Vertex-Softmax significantly improves certified rates and substantially tightens lower bounds across MNIST, Fashion-MNIST, and CIFAR-10 attention models, while consistently matching or outperforming alpha-CROWN and branch-and-bound baselines at a fraction of their cost.
\end{abstract}

\section{Introduction}

Certified verifiers for feedforward networks have matured into practical tools. Transformers remain stubbornly harder---and softmax attention is the reason. The normalized exponential couples all tokens, and existing verifiers must relax it at a cost in certificate tightness: every input where the relaxation is too loose is an input the verifier cannot certify. Closing that gap, making attention verification tighter without sacrificing scalability, is the motivation for this work.

\begin{table*}[!t]
\centering
\setlength{\tabcolsep}{4pt}
\caption{Comparison of verification methods at the softmax interface. The timing column gives representative wall-clock seconds per trial from selected small-attention experiments in Section~\ref{sec:experiments}; it is not the full-block hybrid runtime, which is broken down separately in Appendix Table~\ref{tab:runtime_breakdown_app}.}
\label{tab:method_comparison}
\begin{tabular}{@{}L{0.18\textwidth}L{0.40\textwidth}cL{0.13\textwidth}r@{}}
\toprule
Method & Treatment & Exact? & Corr.? & Sec./trial \\
\midrule
CROWN / auto\_LiRPA & Affine relaxation before contraction with $c$ & No & Partial$^\dagger$ & $0.1$--$0.2$ \\
Wei-LSE & Convex softmax bounds, then contracted with $c$ & No & No & $0.06$--$2.0$ \\
GaLileo-style & Linear softmax relaxation & No & No & $1.8$--$19.2$ \\
$\alpha$-CROWN & Optimized relaxation slopes over full graph & No & Partial$^\ddagger$ & $1.5$--$3.7$ \\
ABCrown-BaB & Branch-and-bound with verifier at nodes & No$^\S$ & Yes & $268$--$395$ \\
\midrule
Vertex-CROWN & Solves $\min_{s\in[\ell,u]} c^\top\!\softmax(s)$ exactly & Yes & No & $0.03$--$0.07$ \\
\bottomrule
\end{tabular}

\vspace{4pt}
\raggedright
$^\dagger$\,Before interval abstraction only.\quad
$^\ddagger$\,Through optimized graph relaxations.\quad
$^\S$\,Not as a local score-box primitive; gains correlations through domain splitting.
\end{table*}

CROWN/LiRPA-style verifiers propagate affine bounds through neural networks at scale~\citep{zhang2018efficient,xu2020automatic}, but attention introduces a normalized exponential map
\begin{equation}
  a(s)=\softmax(s),\qquad
  a_j(s)=\frac{\exp(s_j)}{\sum_{r=1}^K\exp(s_r)}.
\end{equation}
After a score-bounding pass, a verifier often knows only an independent score box
\begin{equation}
  \calB=\prod_{j=1}^K[\ell_j,u_j]
\end{equation}
and a direction $c\in\R^K$ induced by value-side bounds or by a downstream margin. The verifier's task at each attention row thus reduces to optimizing a linear objective over softmax outputs, constrained only by these independent score intervals:
\begin{equation}
  \min_{s\in\calB} c^\top \softmax(s).
  \label{eq:scorebox_problem}
\end{equation}
Generic softmax relaxations first relax the vector map $s\mapsto\softmax(s)$ and only then contract with $c$. We instead solve the direction-dependent scalar problem directly.

The result is exact and small. The minimum over the continuous score box is attained at a vertex of the constraint box. The intuition is that an optimizer trading off mass against cost can always push each coordinate to an endpoint without increasing the objective. After sorting $c$, the optimal vertex puts upper endpoints on the $m$ smallest coefficients and lower endpoints on the rest, for some threshold $m\in\{0,\ldots,K\}$. A na\"ive search examines $2^K$ vertices. We show that sorting the objective coefficients collapses this to exactly $K+1$ candidates---one sort and one prefix-sum sweep. We call this primitive \emph{Vertex-Softmax}.

This also identifies the information surface of the score-box interface. Any sound lower-bound procedure whose inputs are only $(c,\ell,u)$ cannot return a value larger than the exact minimum in~\eqref{eq:scorebox_problem}. Tighter certificates must use information discarded before this interface, such as score correlations, score--value coupling, or a reachable score set smaller than the independent box.

\paragraph{Contributions.}
\begin{itemize}
\item We prove vertex exactness and $K+1$ threshold exactness for~\eqref{eq:scorebox_problem}, giving an $O(K\log K)$ exact solver for the weighted-softmax box problem.
\item We derive a score-box information optimality statement: Vertex-Softmax is the tightest sound bound obtainable from independent score intervals alone.
\item We integrate the primitive into Vertex-CROWN, a sound CROWN-style attention verifier, and show large certified-rate improvements on MNIST and Fashion-MNIST patch-attention models, competitive or improved certificates against alpha-CROWN and branch-and-bound baselines on selected attention blocks, and full-block gains on attention--residual--MLP models with a CROWN-suffix hybrid.
\end{itemize}

\paragraph{Positioning.}
Scalable neural-network verifiers based on affine bound propagation, including CROWN, auto\_LiRPA, and $\alpha,\beta$-CROWN~\citep{zhang2018efficient,xu2020automatic,wang2021betacrown}, have been extended to transformer attention layers~\citep{shi2020robustness,bonaert2021transformers}. Within this line, recent softmax-focused relaxations tighten the vector-valued softmax bounds used inside the verifier. The convex bounds of \citet{wei2023convex} and the GaLileo relaxation of \citet{zhang2024galileo} are representative; both relax softmax rather than solve the direction-dependent score-box objective exactly.

On the other end of the precision--cost spectrum, exact methods such as MILP, MIQCP, SMT, and nonlinear branch-and-bound~\citep{tjeng2019evaluating,katz2019marabou,shi2025genbab} can give complete answers on small instances but do not scale to dense softmax attention. Our contribution sits between these extremes: Vertex-Softmax solves the recurring score-box subproblem exactly in $O(K\log K)$ time, complementing both relaxation-based and search-based approaches. A broader discussion of related work, including LLM-scale statistical and runtime-monitoring approaches, is in Appendix~\ref{app:related_work}. The exactness claim is local to the score-box interface; remaining sources of verifier looseness are discussed in Section~\ref{sec:limitations}. Table~\ref{tab:method_comparison} summarizes the landscape.

\section{Problem Setup}

We now formalize the interface between a CROWN-style score-bounding pass and the softmax layer, and define the scalar optimization problem that Vertex-Softmax solves. Table~\ref{tab:notation} collects the key notation.

\begin{table}[t]
\centering
\setlength{\tabcolsep}{4pt}
\caption{Key notation used throughout the paper.}
\label{tab:notation}
\begin{tabular}{@{}ll@{}}
\toprule
Symbol & Definition \\
\midrule
$\calX$ & Input perturbation set, usually an $\ell_\infty$ pixel box \\
$m(x)$ & Target margin to certify \\
$K$ & Number of keys/tokens in one attention row \\
$s\in\R^K$ & Pre-softmax score vector for one row \\
$a(s)$ & $\softmax(s)$ \\
$c\in\R^K$ & Fixed downstream coefficient vector \\
$\ell,u\in\R^K$ & Certified lower/upper score bounds \\
$\calB$ & Score box $[\ell_1,u_1]\times\cdots\times[\ell_K,u_K]$ \\
$F_c(s)$ & Directional softmax objective $c^\top\!\softmax(s)$ \\
$y_j$ & Exponentiated score $e^{s_j}$ \\
$L_j,U_j$ & Exponentiated bounds $e^{\ell_j},\,e^{u_j}$ \\
$R(y)$ & Linear-fractional objective $(c^\top y)\,/\,(\mathbf{1}^\top y)$ \\
$\tau_m$ & Threshold-vertex value after sorting $c$ \\
$L_{\mathrm{box}}$ & Exact lower bound $\min_{s\in\calB}F_c(s)$ \\
$L_{\mathrm{VC}}$ & Vertex-CROWN margin lower bound \\
$L_{\mathrm{hyb}}$ & Per-target max of CROWN and Vertex-CROWN bounds \\
\bottomrule
\end{tabular}
\end{table}

Let $\calX$ be an input box and let $m:\calX\to\R$ be a scalar margin. Certification means proving $m(x)\ge0$ for all $x\in\calX$, usually by computing a sound lower bound on $\min_{x\in\calX}m(x)$. For one attention row, write
\begin{equation}
  s_i(x)\in\R^K,
  \qquad
  a_i(x)=\softmax(s_i(x)).
\end{equation}
When a backward bound pass reaches this row, the downstream contribution at this interface is summarized by a fixed vector $c_i\in\R^K$, so the row certificate needs a lower bound on $c_i^\top\softmax(s_i(x))$.

The same pass supplies sound elementwise score bounds
\begin{equation}
  \ell_{ij}\le s_{ij}(x)\le u_{ij},\qquad j=1,\ldots,K,\quad x\in\calX,
\end{equation}
forming a product interval $\calB_i=\prod_{j=1}^K[\ell_{ij},u_{ij}]$. Dropping the row index, define
\begin{equation}
  F_c(s)=c^\top\softmax(s)
  =\frac{\sum_{j=1}^K c_j e^{s_j}}{\sum_{j=1}^K e^{s_j}},
  \qquad
  L_{\mathrm{box}}(c,\ell,u)=\min_{s\in\calB}F_c(s).
  \label{eq:lbox}
\end{equation}
The matching upper bound is $U_{\mathrm{box}}(c,\ell,u)=-L_{\mathrm{box}}(-c,\ell,u)$, so we state lower bounds only. Throughout, $K\ge1$, intervals are finite, degenerate intervals are allowed, and $c$ is fixed with respect to $s$. Equation~\eqref{eq:lbox} is exact for the product box, not necessarily for the reachable set $\{s(x):x\in\calX\}\subseteq\calB$.

\section{Exact Score-Box Softmax Optimization}
\label{sec:exact}

Set $y_j=e^{s_j}$, $L_j=e^{\ell_j}$, and $U_j=e^{u_j}$. The map $s\mapsto y$ sends $\calB$ bijectively to $\calB_y=\prod_{j=1}^K[L_j,U_j]$, and
\begin{equation}
  F_c(s)=R(y)\coloneqq
  \frac{\sum_{j=1}^K c_jy_j}{\sum_{j=1}^K y_j}.
  \label{eq:R_def}
\end{equation}
Thus the softmax problem is a bounded linear-fractional optimization problem over a positive box. The vertex property of linear-fractional programs over boxes is classical~\citep{dinkelbach1967nonlinear}; the contribution here is the $K+1$ threshold reduction and its integration into attention verification. Full proofs are in Appendix~\ref{app:proofs}.

\begin{theorem}[Score-box vertex exactness]
\label{thm:vertex}
Let $K\ge1$, $c\in\R^K$, and finite intervals $\ell_j\le u_j$. Then
\begin{equation}
  \min_{s\in\calB}F_c(s)=\min_{v\in\operatorname{Vert}(\calB)}F_c(v),
\end{equation}
and the analogous equality holds for the maximum.
\end{theorem}

The one-coordinate reason is simple: fixing all coordinates except $y_i$, the ratio has form $(c_i y_i+A)/(y_i+D)$ and its derivative has constant sign. Therefore each coordinate can be pushed to an endpoint without increasing the objective.

Sort the coefficients ascending, $c_{(1)}\le\cdots\le c_{(K)}$, and reindex $L,U$ accordingly. For $m\in\{0,\ldots,K\}$ define
\begin{equation}
  y^{(m)}_{(j)}=
  \begin{cases}
    U_{(j)}, & j\le m,\\
    L_{(j)}, & j>m,
  \end{cases}
  \label{eq:threshold_pattern}
\end{equation}
and
\begin{equation}
  \tau_m=R(y^{(m)})=
  \frac{\sum_{j\le m}c_{(j)}U_{(j)}+\sum_{j>m}c_{(j)}L_{(j)}}
       {\sum_{j\le m}U_{(j)}+\sum_{j>m}L_{(j)}}.
  \label{eq:tau_m}
\end{equation}

\begin{theorem}[Threshold exactness]
\label{thm:threshold}
For every $c\in\R^K$ and every finite box $\calB$,
\begin{equation}
  L_{\mathrm{box}}(c,\ell,u)=\min_{s\in\calB}F_c(s)=\min_{m=0,\ldots,K}\tau_m.
\end{equation}
Consequently $L_{\mathrm{box}}(c,\ell,u)$ can be evaluated in $O(K\log K)$ time.
\end{theorem}

The intuition is that the optimizer trades off mass against cost: to minimize the weighted softmax, it maximizes the exponentiated scores on the coordinates with the smallest $c$ coefficients (making them as heavy as possible) and minimizes the scores on the most expensive coordinates. Because $c$ is sorted, the optimal split between ``heavy-cheap'' and ``light-expensive'' coordinates is always a contiguous threshold.

\begin{algorithm}[t]
\caption{Vertex-Softmax threshold solver}
\label{alg:threshold}
\begin{algorithmic}[1]
\STATE \textbf{Input:} coefficients $c$, score lower/upper bounds $\ell,u$
\STATE Sort indices so $c_{(1)}\le\cdots\le c_{(K)}$
\STATE Set $a=\max_j u_j$
\STATE Compute $L_j=\exp(\ell_j-a)$ and $U_j=\exp(u_j-a)$
\STATE Compute prefix sums of $U_{(j)}$ and $c_{(j)}U_{(j)}$
\STATE Compute suffix sums of $L_{(j)}$ and $c_{(j)}L_{(j)}$
\STATE Evaluate all ratios $\tau_m$ in~\eqref{eq:tau_m}
\STATE \textbf{return} $\min_m \tau_m$
\end{algorithmic}
\end{algorithm}

The stability shift in Algorithm~\ref{alg:threshold} rescales every numerator and denominator by the same positive factor, so it preserves all ratios and avoids overflow. A proof-producing interval-arithmetic evaluation path is described in Appendix~\ref{app:numerics}.

\begin{corollary}[Score-box optimality]
\label{thm:optimality}
Let $G(c,\ell,u)$ be any lower-bound procedure depending only on $c$ and the independent intervals. If
\begin{equation}
  G(c,\ell,u)\le F_c(s)\qquad\forall s\in\calB,
\end{equation}
then
\begin{equation}
  G(c,\ell,u)\le L_{\mathrm{box}}(c,\ell,u).
\end{equation}
Vertex-Softmax therefore exhausts the information available at the score-box interface: any tighter certificate must exploit score correlations, score--value coupling, or a reachable score set smaller than the independent box.
\end{corollary}

\section{Vertex-CROWN Attention Bound}
\label{sec:vertex_crown}

The previous section provides an exact solver for a single softmax row given fixed coefficients and score intervals. To turn this into a working verifier, we need to extract those quantities from a CROWN/LiRPA backward pass and compose the row-level certificates into an end-to-end margin bound. Vertex-CROWN is the resulting interface. Figure~\ref{fig:vertex-softmax-pipeline} illustrates where the Vertex-Softmax primitive fits in the overall verification pipeline.

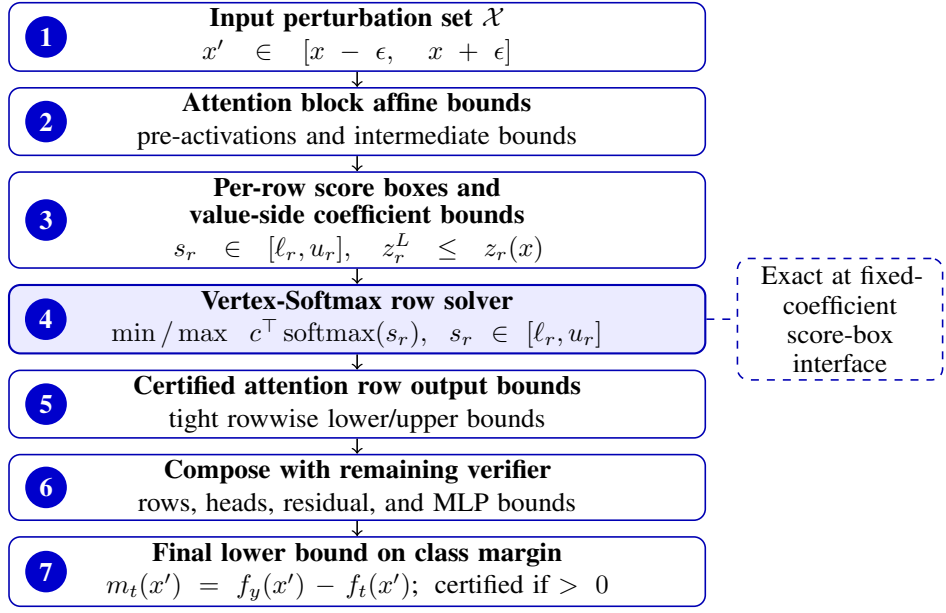
\begin{figure*}[t]
\centering
\begin{tikzpicture}[
    font=\normalsize,
    flow/.style={->, line width=0.6pt},
    stepbox/.style={
        draw=blue!70!black,
        rounded corners=1.5mm,
        line width=0.6pt,
        fill=white,
        minimum width=9.2cm,
        minimum height=0.8cm,
        text width=8.4cm,
        align=center,
        inner sep=2pt
    },
    keybox/.style={
        stepbox,
        fill=blue!8,
        line width=0.8pt,
        minimum height=0.9cm
    },
    badge/.style={
        circle,
        fill=blue!80!black,
        text=white,
        font=\normalsize\bfseries,
        minimum size=5.5mm,
        inner sep=0pt
    },
    callout/.style={
        draw=blue!70!black,
        dashed,
        rounded corners=1.5mm,
        line width=0.6pt,
        fill=white,
        align=center,
        text width=2.5cm,
        inner sep=3pt
    }
]

\node[stepbox] (s1)
{\textbf{Input perturbation set $\calX$}\\[0.5mm]
$x' \in [x-\epsilon,\; x+\epsilon]$};
\node[badge] at ([xshift=5mm]s1.west) {1};

\node[stepbox, below=2mm of s1] (s2)
{\textbf{Attention block affine bounds}\\[0.5mm]
pre-activations and intermediate bounds};
\node[badge] at ([xshift=5mm]s2.west) {2};

\node[stepbox, below=2mm of s2] (s3)
{\textbf{Per-row score boxes and}\\[-0.2mm]
\textbf{value-side coefficient bounds}\\[0.5mm]
$s_r \in [\ell_r,u_r],\; z_r^L \le z_r(x)$};
\node[badge] at ([xshift=5mm]s3.west) {3};

\node[keybox, below=2mm of s3] (s4)
{\textbf{Vertex-Softmax row solver}\\[0.5mm]
$\min/\max\; c^\top \softmax(s_r),\; s_r \in [\ell_r,u_r]$};
\node[badge] at ([xshift=5mm]s4.west) {4};

\node[callout, right=4mm of s4] (callout)
{Exact at fixed-coefficient\\ score-box interface};
\draw[blue!70!black, dashed, line width=0.6pt] (s4.east) -- (callout.west);

\node[stepbox, below=2mm of s4] (s5)
{\textbf{Certified attention row output bounds}\\[0.5mm]
tight rowwise lower/upper bounds};
\node[badge] at ([xshift=5mm]s5.west) {5};

\node[stepbox, below=2mm of s5] (s6)
{\textbf{Compose with remaining verifier}\\[0.5mm]
rows, heads, residual, and MLP bounds};
\node[badge] at ([xshift=5mm]s6.west) {6};

\node[stepbox, below=2mm of s6] (s7)
{\textbf{Final lower bound on class margin}\\[0.5mm]
$m_t(x') = f_y(x') - f_t(x')$;\; certified if $> 0$};
\node[badge] at ([xshift=5mm]s7.west) {7};

\draw[flow] (s1.south) -- (s2.north);
\draw[flow] (s2.south) -- (s3.north);
\draw[flow] (s3.south) -- (s4.north);
\draw[flow] (s4.south) -- (s5.north);
\draw[flow] (s5.south) -- (s6.north);
\draw[flow] (s6.south) -- (s7.north);

\end{tikzpicture}
\caption{
Integration of Vertex-Softmax into the verification pipeline. Input perturbations are propagated to independent score boxes and value-side coefficient lower bounds for each attention row. Vertex-Softmax computes exact directional softmax bounds at this rowwise score-box interface, which are then composed with the remaining verifier to obtain a certified lower bound on the final class margin.
}
\label{fig:vertex-softmax-pipeline}
\end{figure*}

Suppose a scalar margin satisfies
\begin{equation}
  m(x)\ge b+\sum_i\sum_{j=1}^K a_i(x)_j z_{ij}(x)
  \qquad \forall x\in\calX,
  \label{eq:vc_margin_interface}
\end{equation}
where $a_i(x)=\softmax(s_i(x))$. Assume sound score and value-side lower bounds
\begin{equation}
  \ell_{ij}\le s_{ij}(x)\le u_{ij},
  \qquad
  z_{ij}^L\le z_{ij}(x)
  \qquad \forall x\in\calX.
  \label{eq:vc_score_value_bounds}
\end{equation}
For proof references we also name the two parts separately:
\begin{equation}
  \ell_{ij}\le s_{ij}(x)\le u_{ij}
  \qquad \forall x\in\calX,
  \label{eq:vc_score_bounds}
\end{equation}
\begin{equation}
  z_{ij}^L\le z_{ij}(x)
  \qquad \forall x\in\calX.
  \label{eq:vc_value_bounds}
\end{equation}
Let $\calB_i=\prod_j[\ell_{ij},u_{ij}]$ and
\begin{equation}
  L_i=L_{\mathrm{box}}(z_i^L,\ell_i,u_i)
  =\min_{s_i\in\calB_i}\sum_{j=1}^K\softmax(s_i)_j z_{ij}^L.
  \label{eq:vc_row_bound}
\end{equation}
The layer bound is
\begin{equation}
  L_{\mathrm{VC}}=b+\sum_i L_i.
  \label{eq:vc_bound}
\end{equation}

\begin{algorithm}[t]
\caption{Vertex-CROWN attention bound}
\label{alg:vertex_crown}
\begin{algorithmic}[1]
\STATE \textbf{Input:} input set $\calX$, scalar margin $m$, attention row scores $s_i$
\STATE Use CROWN/LiRPA to obtain sound row score bounds $(\ell_i,u_i)$
\STATE Obtain lower value coefficients $z_i^L$ and affine lower term $b$ satisfying~\eqref{eq:vc_margin_interface}--\eqref{eq:vc_value_bounds}
\FOR{each query row $i$}
  \STATE $L_i \leftarrow L_{\mathrm{box}}(z_i^L,\ell_i,u_i)$ using Algorithm~\ref{alg:threshold}
\ENDFOR
\STATE \textbf{return} $L_{\mathrm{VC}}=b+\sum_i L_i$
\end{algorithmic}
\end{algorithm}

\begin{proposition}[Soundness of Vertex-CROWN]
\label{prop:vc_soundness}
Under~\eqref{eq:vc_margin_interface}, \eqref{eq:vc_score_bounds}, and~\eqref{eq:vc_value_bounds}, Algorithm~\ref{alg:vertex_crown} returns a sound lower bound:
\begin{equation}
  L_{\mathrm{VC}}\le m(x)\qquad \forall x\in\calX.
\end{equation}
\end{proposition}

In practice, if $L_{\mathrm{VC}}>0$, the model's prediction is provably unchanged under every perturbation in~$\calX$. Tighter lower bounds certify more inputs and give a more informative picture of the model's robustness.

\paragraph{Self-attention and suffix instantiation.}
For the ViT-style blocks in the experiments, token embeddings $H(x)$ produce affine queries, keys, and values
\begin{equation}
  Q_i^h=W_Q^hH_i+b_Q^h,
  \qquad
  K_j^h=W_K^hH_j+b_K^h,
  \qquad
  V_j^h=W_V^hH_j+b_V^h,
\end{equation}
with scores
\begin{equation}
  s_{ij}^h(x)=d_h^{-1/2}\langle Q_i^h(x),K_j^h(x)\rangle+\mu_{ij}^h.
\end{equation}
The score intervals can be obtained by CROWN/McCormick bounds on this scalar score module, or by the interval-product fallback in Appendix~\ref{app:vc_details}.

For a target margin, let $H_i^+(x)=H_i(x)+b_O+\sum_h W_O^h\sum_j a_{ij}^h(x)V_j^h(x)$ be the post-attention residual state. If the downstream suffix gives
\begin{equation}
  m_t(x)\ge \beta_t+\sum_i\gamma_{ti}^{\top}H_i^+(x),
\end{equation}
then the induced row coefficients are lower bounds of $(W_O^h)^\top\gamma_{ti}$ contracted with $V_j^h(x)$, and the Vertex path has form
\begin{equation}
  L_{\mathrm{VC},t}=b_t'+\sum_{h,i}L_{\mathrm{box}}(c_{tih:},\ell_{ih:},u_{ih:}).
\end{equation}
The target-wise hybrid uses
\begin{equation}
  L_{\mathrm{hyb},t}=\max\{L_{\mathrm{CROWN},t},L_{\mathrm{VC},t}\},
  \label{eq:hybrid_bound}
\end{equation}
which is sound because both terms lower-bound the same target margin. The two paths complement each other: direct CROWN propagates tighter affine relaxations through the full computation graph, while Vertex-CROWN solves the softmax subproblem exactly but relies on looser value-side and suffix bounds. Taking the per-target maximum captures the best of both without sacrificing soundness. Details, coefficient formulas, and the per-target algorithm are in Appendix~\ref{app:vc_details}. The oracle cost is $O(BTHRK\log K)$ for batch size $B$, targets $T$, heads $H$, query rows $R$, and keys $K$, excluding score/value/suffix bound construction.

\section{Experiments}
\label{sec:experiments}

The experiments test whether the exact score-box primitive sharpens attention certificates at the verifier interface. We progress from validating the solver itself, to scalable softmax comparisons on synthetic score boxes, to real-data patch-attention models, to full attention--residual--MLP blocks, and finally to comparisons against stronger optimized and search-based baselines.

Certified-rate denominators differ by benchmark. Synthetic rows are computed over all generated trials. Patch-attention rows in Table~\ref{tab:mnist_patch_attention} are conditional on clean-correct examples. Full-block rows in Table~\ref{tab:full_mlp_nonmnist} are end-to-end certified accuracies over the listed evaluation prefixes, with clean-incorrect examples counted as uncertified. The exactness claim applies only to the fixed-coefficient independent score-box subproblem, not to end-to-end transformer verification. The reported lower bound is the mean margin lower bound; and the gap is best attack margin minus certified lower bound. All reported verifier tables use standard floating-point PyTorch/auto\_LiRPA unless explicitly marked otherwise; the mathematical soundness statements are real-arithmetic statements, with interval-oracle checks in Appendix~\ref{app:numerics}. Protocol and numerical-status records are in Appendix Tables~\ref{tab:protocol_app} and~\ref{tab:numerical_status_app}.

\paragraph{Solver and scalable softmax checks.}
The threshold solver agrees with exhaustive vertex enumeration up to numerical tolerance for $K\le16$ and remains practical at $K=512$; Appendix Table~\ref{tab:threshold_runtime_app} and Figure~\ref{fig:threshold_runtime_app} give the timing details. Table~\ref{tab:scalable_sweep} shows the main scalable-softmax comparison. Vertex-Softmax gives the best mean lower bound and smallest attack gap for every tested $K\in\{4,8,16,32,64,128\}$, while also being $3$--$24\times$ faster than Wei-LSE (e.g., $0.08$s vs.\ $2.0$s at $K=128$).

\begin{figure*}[t]
  \centering
  \includegraphics[width=0.95\textwidth]{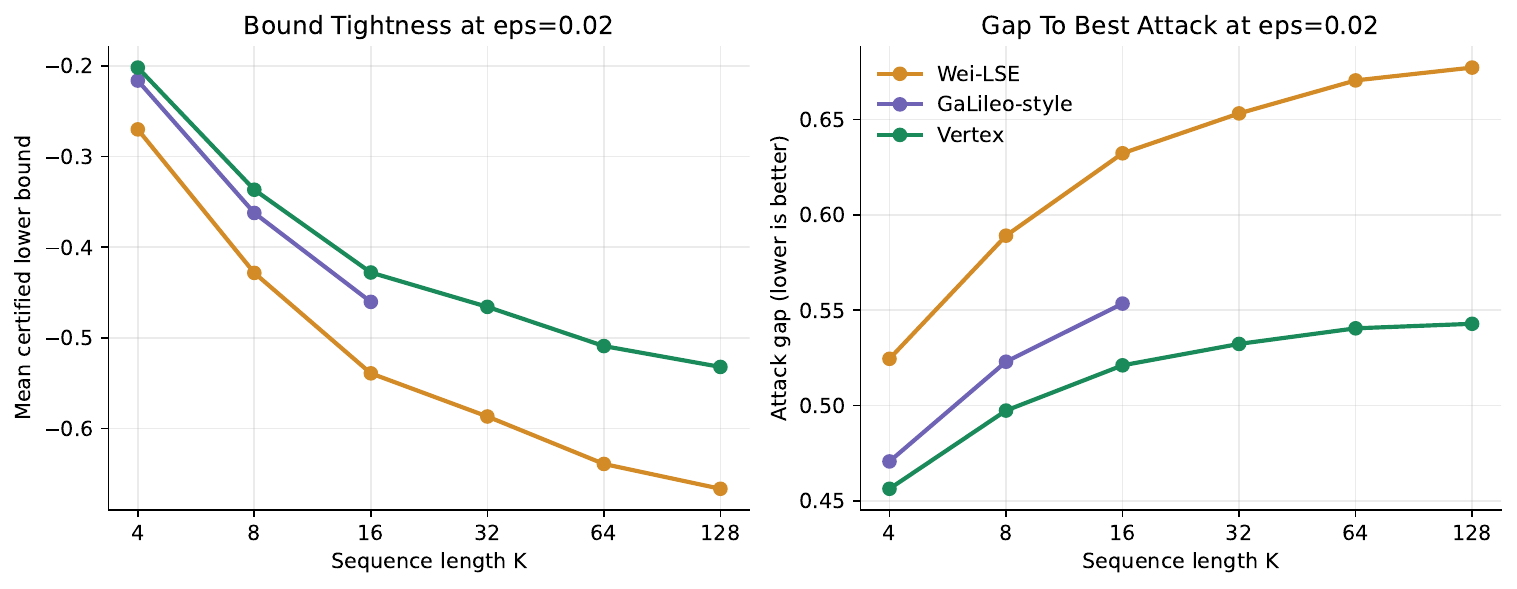}
  \caption{Scalable attention sweep at $\epsilon=0.02$. Vertex-Softmax gives the highest mean certified lower bound and the smallest gap to the best attack across the tested sequence lengths. GaLileo-style implementation baseline results for $K\le16$ are in Appendix Table~\ref{tab:galileo_baseline_app}.}
  \label{fig:scalable_sweep}
\end{figure*}

\begin{table*}[t]
\centering
\setlength{\tabcolsep}{2pt}
\caption{Scalable sweep at $\epsilon=0.02$ over three seeds and 600 total trials per row. Vertex-Softmax gives the best mean lower bound and smallest attack gap at every tested $K$. Runtime is aggregate wall time. An independent GaLileo-style implementation baseline for $K\le16$ is in Appendix Table~\ref{tab:galileo_baseline_app} (the original GaLileo code is not publicly available).}
\label{tab:scalable_sweep}
\begin{tabular}{rlrrrr}
\toprule
$K$ & Method & Cert. rate & Mean lower & Gap & Time (s) \\
\midrule
4   & Wei-LSE       & 0.128 & -0.2702 & 0.5245 & 0.1080 \\
4   & Vertex        & 0.180 & -0.2020 & 0.4563 & 0.0298 \\
8   & Wei-LSE       & 0.023 & -0.4284 & 0.5891 & 0.0645 \\
8   & Vertex        & 0.053 & -0.3367 & 0.4973 & 0.0299 \\
16  & Wei-LSE       & 0.000 & -0.5392 & 0.6323 & 0.1012 \\
16  & Vertex        & 0.007 & -0.4280 & 0.5211 & 0.0302 \\
32  & Wei-LSE       & 0.000 & -0.5868 & 0.6532 & 0.8202 \\
32  & Vertex        & 0.000 & -0.4659 & 0.5323 & 0.1391 \\
64  & Wei-LSE       & 0.000 & -0.6391 & 0.6705 & 1.4836 \\
64  & Vertex        & 0.000 & -0.5091 & 0.5405 & 0.1133 \\
128 & Wei-LSE       & 0.000 & -0.6665 & 0.6772 & 2.0094 \\
128 & Vertex        & 0.000 & -0.5321 & 0.5428 & 0.0838 \\
\bottomrule
\end{tabular}
\end{table*}

\paragraph{Image attention certificates.}
Table~\ref{tab:mnist_patch_attention} evaluates small real-data patch-attention classifiers. The strongest signal appears in high-uncertainty regimes: at $\epsilon=0.03$, Vertex-CROWN raises binary MNIST certification from $20.3\%$ to $77.8\%$ with 16 tokens and from $46.5\%$ to $92.5\%$ with 49 tokens. These are regimes where standard CROWN leaves most inputs uncertified; replacing the softmax relaxation with the exact solver recovers the majority of them. The 10-class setting is harder, but Vertex-CROWN still improves both certified rate and mean lower bound at $\epsilon\in\{0.02,0.03\}$; Appendix Table~\ref{tab:bootstrap_ci_app} reports seed-resampling stability checks for the main rate improvements.

\begin{table*}[t]
\centering
\setlength{\tabcolsep}{3pt}
\caption{Real-data MNIST patch-attention certificates. Binary rows use classes $0$ vs.~$1$; 10-class rows use all MNIST classes. The 16-token model uses $7\times7$ patches, and the 49-token model uses $4\times4$ patches. Certification rates are conditional on clean-correct examples: verification is run on up to the protocol's listed $N$ correctly classified test examples per seed, so these rates are not end-to-end certified accuracies. Results are averaged over three seeds for the 16-token rows and two seeds for the 49-token row.}
\label{tab:mnist_patch_attention}
\begin{tabular}{llrrrr}
\toprule
Benchmark & $\epsilon$ & CROWN cert. & Vertex-CROWN cert. & CROWN lower & Vertex lower \\
\midrule
Binary, 16 tokens & 0.02 & 0.738 & 0.916 &  1.915 &  4.552 \\
Binary, 16 tokens & 0.03 & 0.203 & 0.778 & -8.263 &  2.426 \\
Binary, 16 tokens & 0.05 & 0.002 & 0.260 & -74.505 & -2.142 \\
Binary, 49 tokens & 0.02 & 0.840 & 0.972 &  3.788 &  6.482 \\
Binary, 49 tokens & 0.03 & 0.465 & 0.925 & -3.354 &  4.934 \\
10-class, 16 tokens & 0.02 & 0.033 & 0.129 & -8.931 & -2.895 \\
10-class, 16 tokens & 0.03 & 0.001 & 0.012 & -30.785 & -7.375 \\
\bottomrule
\end{tabular}
\end{table*}

\paragraph{Full attention--residual--MLP blocks.}
Table~\ref{tab:full_mlp_nonmnist} tests whether the primitive's gains survive a nonlinear suffix. The CROWN-suffix hybrid improves Fashion-MNIST end-to-end certified rates and mean lower bounds across the robustness curve: at $\epsilon=0.02$ it raises certification from $5.6\%$ to $14.0\%$ and improves the mean lower bound from $-7.784$ to $-1.382$; at $\epsilon=0.03$ the mean lower bound improves by over $33$ points. CIFAR-10 grayscale, included as a low-accuracy stress setting, shows the same tightness direction. Appendix Tables~\ref{tab:paired_nonmnist_app}, \ref{tab:bootstrap_ci_app}, and~\ref{tab:runtime_breakdown_app} give paired deltas, seed-resampling stability checks, and runtime accounting. The sort/sweep oracle is negligible relative to score, value, and suffix-bound construction.

\paragraph{Stronger baselines.}
\begin{figure*}[t]
  \centering
  \includegraphics[width=0.95\textwidth]{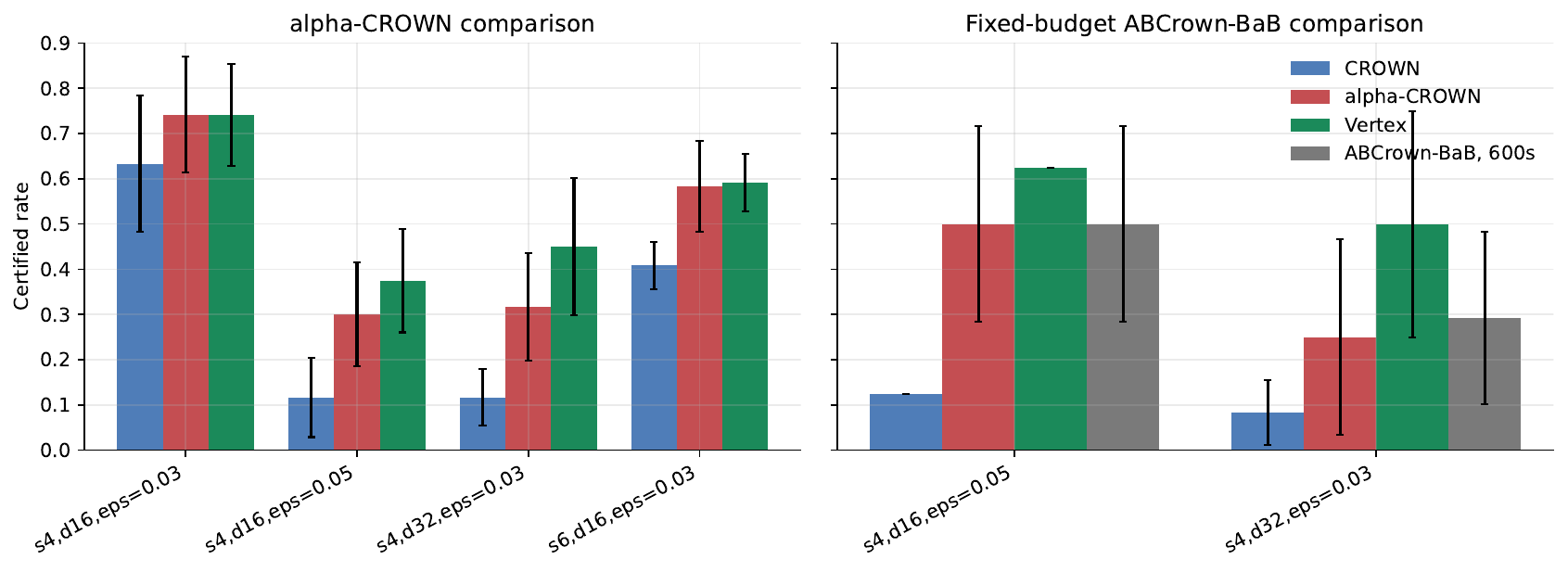}
  \caption{Comparison against alpha-CROWN and fixed-budget ABCrown-BaB on selected small attention blocks. Vertex-CROWN is competitive with alpha-CROWN across settings and outperforms it in higher-epsilon and higher-dimensional configurations, while running $30$--$100\times$ faster. ABCrown-BaB uses a 600s per-instance budget; unknown, timeout, OOM, and internal-error outcomes are counted as uncertified.}
  \label{fig:strong_baselines}
\end{figure*}

Against alpha-CROWN and fixed-budget ABCrown-BaB on selected small attention blocks, Vertex-CROWN matches or exceeds alpha-CROWN at $K=4,d=16,\epsilon=0.03$ and $K=6,d=16,\epsilon=0.03$, outperforms it in higher-epsilon and higher-dimensional $K=4$ settings, and trails it on the residual block at $\epsilon=0.03$ where optimizable slopes have more room to help. With a 600 second ABCrown-BaB budget, Vertex-CROWN certifies $62.5\%$ versus $50.0\%$ on $K=4,d=16,\epsilon=0.05$ and $50.0\%$ versus $29.2\%$ on $K=4,d=32,\epsilon=0.03$. Notably, Vertex-CROWN achieves these results in $0.03$--$0.07$s per trial, compared to $1.5$--$3.7$s for alpha-CROWN and $268$--$395$s for ABCrown-BaB. Figure~\ref{fig:strong_baselines} and Appendix Tables~\ref{tab:paired_strong_baseline_app} and~\ref{tab:selected_strong_baselines_app} give paired, aggregate, and fixed-budget details.

\paragraph{Practical guidance.}
Vertex-CROWN is most effective when the softmax relaxation is the binding source of certificate looseness: high-uncertainty regimes with wider score boxes, settings where alpha-CROWN's slope optimization has limited room, and architectures where the attention layer dominates the verification gap. When suffix or value-bound looseness dominates instead, the hybrid strategy can still recover the better of the two paths, but full-block use incurs the substantial bound-construction cost quantified in Appendix Table~\ref{tab:runtime_breakdown_app}.

\begin{table*}[t]
\centering
\setlength{\tabcolsep}{2pt}
\caption{Full attention--residual--MLP block results. Cert. entries are end-to-end certified accuracies, reported as CROWN / Hybrid over the first 100 evaluation examples per seed with clean-incorrect examples counted as uncertified. Lower columns show CROWN / Hybrid mean margin lower bounds on the clean-correct examples from the same evaluation prefix. $\Delta$ is the paired mean lower-bound improvement of Hybrid over CROWN. All rows use five seeds.}
\label{tab:full_mlp_nonmnist}
\begin{tabular}{llrrrrr}
\toprule
Dataset & Setting & $\epsilon$ & Clean & Cert. & Lower & $\Delta$ \\
\midrule
Fashion-MNIST & $d=32,h=4,m=64$ & 0.01  & 0.708 & 0.456 / 0.484 &   1.153 /   1.368 &  0.215 \\
Fashion-MNIST & $d=32,h=4,m=64$ & 0.02  & 0.708 & 0.056 / 0.140 &  -7.784 /  -1.382 &  6.402 \\
Fashion-MNIST & $d=32,h=4,m=64$ & 0.03  & 0.708 & 0.000 / 0.058 & -38.191 /  -4.583 & 33.608 \\
Fashion-MNIST & $d=16,h=2,m=64$ & 0.02  & 0.672 & 0.096 / 0.252 &  -5.228 /  -0.620 &  4.608 \\
CIFAR-10 gray & $d=32,h=4,m=64$ & 0.005 & 0.346 & 0.126 / 0.150 &  -0.280 /  -0.098 &  0.182 \\
CIFAR-10 gray & $d=32,h=4,m=64$ & 0.01  & 0.346 & 0.004 / 0.026 &  -6.041 /  -1.520 &  4.521 \\
\bottomrule
\end{tabular}
\end{table*}

\section{Limitations}
\label{sec:limitations}

Vertex-Softmax eliminates one source of looseness, the softmax relaxation inside the score-box subproblem, but the full verification pipeline has several other sources of slack that the primitive does not address. Understanding these is important both for interpreting the experimental results and for identifying where future work can have the most impact.

Vertex-Softmax is exact only after the verifier has reduced an attention row to fixed coefficients and independent score intervals. The remaining looseness comes from the surrounding verifier, not from the weighted-softmax box oracle. First, independent score boxes discard score correlations; if those correlations rule out the worst threshold vertex, exact box optimization is still conservative for the reachable set. Second, value lower bounds discard score--value coupling; nonnegative softmax weights preserve soundness, but weak value bounds can dominate the final certificate. Third, row-wise minimization can combine worst cases that no single input attains simultaneously. Fourth, nonlinear suffixes require additional relaxations, so the CROWN-suffix hybrid in Table~\ref{tab:full_mlp_nonmnist} is tighter but slower than direct CROWN and still depends on suffix-bound quality. The experiments are one-block ViT-style models rather than deep transformers with layer normalization and many interacting blocks. Appendix Table~\ref{tab:slack_decomposition_app} gives controlled slack-decomposition diagnostics, and Appendix Table~\ref{tab:runtime_breakdown_app} separates oracle cost from surrounding bound-construction cost.

\section{Conclusion}

Vertex-Softmax exactly solves $\min_{s\in\calB}c^\top\softmax(s)$ at the independent score-box interface. The optimum is attained by one of $K+1$ sorted threshold vertices, giving an exact $O(K\log K)$ primitive that replaces objective-agnostic softmax relaxations with a direction-aware row solver. Integrated into Vertex-CROWN, it produces large certified-rate improvements on patch-attention and full attention--residual--MLP models, and is competitive with or improves upon alpha-CROWN and branch-and-bound baselines on selected attention blocks.

The threshold structure also clarifies where further progress must come from. Because Vertex-Softmax exhausts the information at the independent score-box interface, tighter end-to-end certificates require exploiting score correlations, score--value coupling, joint row composition, or tighter nonlinear suffix bounds. Extending the exact-optimization perspective to these richer interfaces, for instance solving correlated score polytopes or jointly optimizing over scores and values, is a natural next step toward closing the gap between local softmax exactness and global verifier tightness.

\clearpage
\bibliographystyle{plainnat}
\bibliography{references}

\appendix

\section{Extended Related Work}
\label{app:related_work}

At LLM scale, the verification literature shifts toward different kinds of guarantees. Erase-and-Check certifies wrapper behavior against bounded adversarial prompting~\citep{kumar2024certifying}; distributional and conformal approaches give statistical guarantees over prompt distributions~\citep{chaudhary2025counterfactual,cherian2024conformal}; empirical suites such as PromptBench and SafetyBench search for failures but are not sound certificates~\citep{zhu2024promptbench,zhang2024safetybench}; and runtime monitors such as AgentSpec, RvLLM, and ProbGuard enforce explicit trace or action predicates during deployment~\citep{wang2025agentspec,zhang2025rvllm,wang2025probguard}. These approaches are practically important for frontier LLM systems, but they certify wrappers, distributions, tests, or traces rather than the worst-case white-box semantics of the transformer computation graph. Vertex-Softmax is complementary to all of these: it contributes a local exact primitive at the score-box interface, not an end-to-end system for any of the above guarantee types.

For transformers specifically, the exact MIQCP result of \citet{liao2023sparsemax} applies to sparsemax attention rather than standard softmax attention. Set-based reachability methods~\citep{tran2020nnv} can provide complete answers but have not yet been demonstrated at dense softmax scale.

\section{Proof Details}
\label{app:proofs}

This appendix proves Theorems~\ref{thm:vertex} and~\ref{thm:threshold}, derives Corollary~\ref{thm:optimality} from them, and finally establishes the soundness statement in Proposition~\ref{prop:vc_soundness}. Throughout we use the change of variables $y_j=e^{s_j}$, $L_j=e^{\ell_j}$, $U_j=e^{u_j}$, so that $F_c(s)=R(y)$ with $R$ as in~\eqref{eq:R_def}, and $y$ ranges over the positive box $\calB_y=\prod_{j=1}^K[L_j,U_j]\subset(0,\infty)^K$. Because the map $s\mapsto y$ is a coordinatewise bijection that sends vertices of $\calB$ to vertices of $\calB_y$, every statement below transfers to $F_c$ on $\calB$ without change.

\subsection{Proof of Theorem~\ref{thm:vertex}}

We show that $R$ attains its minimum over $\calB_y$ at a vertex of $\calB_y$. The maximum statement follows by applying the same argument with $c$ replaced by $-c$.

\begin{proof}[Proof of Theorem~\ref{thm:vertex}]
\emph{Step 1: one-dimensional reduction.}
If $K=1$, then $R(y)=c_1$ is constant on $\calB_y$, so every point is a minimizer; in particular, both endpoints, which are vertices, are minimizers. Thus the claim is immediate. Assume $K\ge 2$.

Fix any index $i\in\{1,\ldots,K\}$ and any choice of values $y_j\in[L_j,U_j]$ for $j\ne i$. Write
\begin{equation}
  A=\sum_{j\ne i}c_jy_j,\qquad
  D=\sum_{j\ne i}y_j.
\end{equation}
Since $y_j>0$ for every $j$, we have $D>0$. Along the $i$th coordinate, $R$ takes the one-dimensional form
\begin{equation}
  \phi_i(y_i)=\frac{c_iy_i+A}{y_i+D},
  \qquad y_i\in[L_i,U_i].
  \label{eq:app_phi_i}
\end{equation}
Differentiating,
\begin{equation}
  \phi_i'(y_i)
  =
  \frac{c_i(y_i+D)-(c_iy_i+A)}{(y_i+D)^2}
  =
  \frac{c_iD-A}{(y_i+D)^2}.
\end{equation}
The denominator is strictly positive and the numerator is independent of $y_i$, so $\phi_i$ is either monotone on $[L_i,U_i]$ or constant there. In every case, at least one of $L_i$ or $U_i$ attains $\min_{[L_i,U_i]}\phi_i$.

\emph{Step 2: coordinate-by-coordinate pushing.}
Since $R$ is continuous and $\calB_y$ is compact, the minimum $R^\star=\min_{\calB_y}R$ is attained. Pick any minimizer $y^\star\in\calB_y$ and define a finite sequence $y^{(0)},y^{(1)},\ldots,y^{(K)}$ by $y^{(0)}=y^\star$ and, for $i=1,\ldots,K$, choose $y^{(i)}_i$ to be an endpoint of $[L_i,U_i]$ that minimizes~\eqref{eq:app_phi_i} with all other coordinates fixed at their values in $y^{(i-1)}$, setting $y^{(i)}_j=y^{(i-1)}_j$ for $j\ne i$. By Step~1, each replacement satisfies $R(y^{(i)})\le R(y^{(i-1)})$. Iterating,
\begin{equation}
  R^\star\le R(y^{(K)})\le R(y^{(0)})=R^\star,
\end{equation}
so $R(y^{(K)})=R^\star$. By construction $y^{(K)}_j\in\{L_j,U_j\}$ for every $j$, and therefore $y^{(K)}$ is a vertex of $\calB_y$. Pulling back through $s\mapsto y$ yields a vertex of $\calB$ attaining the minimum of $F_c$.
\end{proof}

Degenerate intervals $\ell_j=u_j$ require no modification: the corresponding coordinate is forced to its common endpoint, which is still a vertex.

\subsection{Proof of Theorem~\ref{thm:threshold}}

By Theorem~\ref{thm:vertex}, it suffices to minimize $R$ over the $2^K$ vertices of $\calB_y$. We show that at least one of those minimizers has threshold form in the sense of~\eqref{eq:threshold_pattern}, which narrows attention to the $K+1$ candidates $y^{(0)},\ldots,y^{(K)}$.

Reindex so that $c_1\le c_2\le\cdots\le c_K$, and denote the associated bounds by $L_j,U_j$. This is the sorted order $c_{(j)}$, $L_{(j)}$, $U_{(j)}$ used in the main text, and the value of $R$ is invariant under this reindexing.

\begin{proof}[Proof of Theorem~\ref{thm:threshold}]
\emph{Step 1: linearization at the optimum.}
For every $\rho\in\R$ and every $y\in\calB_y$, the common denominator $S(y)=\sum_j y_j$ is strictly positive, and
\begin{equation}
  R(y)-\rho
  =
  \frac{\sum_{j=1}^K(c_j-\rho)y_j}{S(y)}.
  \label{eq:app_dinkelbach}
\end{equation}
Let $y^\star\in\calB_y$ be any minimizer of $R$ and set $\rho=R(y^\star)$. Then~\eqref{eq:app_dinkelbach} at $y^\star$ gives $\sum_j(c_j-\rho)y^\star_j=0$. For every $y\in\calB_y$, minimality of $y^\star$ gives $R(y)\ge\rho$, so
\begin{equation}
  \sum_{j=1}^K(c_j-\rho)y_j
  \ge
  0
  =
  \sum_{j=1}^K(c_j-\rho)y^\star_j.
  \label{eq:app_linear}
\end{equation}
Thus $y^\star$ minimizes the linear functional $y\mapsto\sum_j(c_j-\rho)y_j$ over $\calB_y$.

\emph{Step 2: a linear minimizer has threshold form.}
The linear problem in~\eqref{eq:app_linear} separates across coordinates. For each $j$,
\begin{equation}
  y^\star_j\in
  \operatorname*{arg\,min}_{y_j\in[L_j,U_j]}(c_j-\rho)y_j,
\end{equation}
and so necessarily
\begin{equation}
  y^\star_j=
  \begin{cases}
    U_j, & c_j<\rho,\\
    L_j, & c_j>\rho,\\
    \text{any point in }[L_j,U_j], & c_j=\rho.
  \end{cases}
  \label{eq:app_coord_rule}
\end{equation}
Define $m^-=|\{j:c_j<\rho\}|$ and $m^+=|\{j:c_j\le\rho\}|$. Since the coefficients are sorted, $c_j<\rho$ exactly for $j\le m^-$ and $c_j>\rho$ exactly for $j>m^+$. On the tied block $\{m^-+1,\ldots,m^+\}$ the linear functional contributes zero, so replacing $y^\star$ on that block by any other feasible values preserves $\sum_j(c_j-\rho)y_j=0$ and hence preserves $R(y)=\rho$ by~\eqref{eq:app_dinkelbach}. In particular, define a threshold vertex $\tilde y$ by
\begin{equation}
  \tilde y_j=
  \begin{cases}
    U_j, & j\le m^+,\\
    L_j, & j>m^+.
  \end{cases}
\end{equation}
Then $\tilde y$ is exactly the threshold pattern~\eqref{eq:threshold_pattern} with $m=m^+$, and $R(\tilde y)=R(y^\star)=\min_{\calB_y}R$.

\emph{Step 3: $K+1$ candidates suffice.}
By Step~2, some index $m\in\{0,1,\ldots,K\}$ satisfies $R(y^{(m)})=\min_{\calB_y}R$. Therefore
\begin{align}
  \min_{s\in\calB}F_c(s)
  &=
  \min_{y\in\calB_y}R(y) \nonumber\\
  &=
  \min_{m=0,\ldots,K}R\bigl(y^{(m)}\bigr)
  =
  \min_{m=0,\ldots,K}\tau_m,
\end{align}
which is the claimed exactness identity.

\emph{Step 4: complexity.}
Sorting $c$ takes $O(K\log K)$. With the sorted order fixed, a left-to-right pass computes prefix sums of $U_{(j)}$ and $c_{(j)}U_{(j)}$, and a right-to-left pass computes suffix sums of $L_{(j)}$ and $c_{(j)}L_{(j)}$, each in $O(K)$ time. Given these four arrays, every $\tau_m$ in~\eqref{eq:tau_m} is a ratio of two $O(1)$ lookups, so the full sweep over $m=0,\ldots,K$ costs $O(K)$. The total is therefore $O(K\log K)$, dominated by the sort.
\end{proof}

For numerical stability, replacing $(\ell_j,u_j)$ by $(\ell_j-a,u_j-a)$ rescales every $y_j$ by $e^{-a}$, and therefore scales both the numerator and the denominator of every $\tau_m$ by the same factor. Choosing $a=\max_j u_j$ enforces $y_j\in(0,1]$ for all $j$ and prevents overflow.

\subsection{Proof of Corollary~\ref{thm:optimality}}

\begin{proof}[Proof of Corollary~\ref{thm:optimality}]
Let $G$ be any sound score-box-only lower-bound procedure. By soundness,
\begin{equation}
  G(c,\ell,u)\le F_c(s)\qquad\forall\,s\in\calB.
\end{equation}
The right-hand side does not depend on $G$, so we may take the infimum over $s\in\calB$:
\begin{equation}
  G(c,\ell,u)\le \inf_{s\in\calB}F_c(s).
\end{equation}
The box $\calB$ is compact and $F_c$ is continuous, so the infimum is attained and equals $\min_{s\in\calB}F_c(s)$. By Theorem~\ref{thm:threshold},
\begin{equation}
  \min_{s\in\calB}F_c(s)=\min_{m=0,\ldots,K}\tau_m=L_{\mathrm{box}}(c,\ell,u).
\end{equation}
Combining the displays gives $G(c,\ell,u)\le L_{\mathrm{box}}(c,\ell,u)$.
\end{proof}

Algorithm~\ref{alg:threshold} is itself a concrete sound score-box-only procedure whose output equals $L_{\mathrm{box}}(c,\ell,u)$, so the inequality in Corollary~\ref{thm:optimality} is attained. Among all sound procedures whose only inputs are $c$ and the independent intervals, Vertex-Softmax is therefore pointwise optimal at every $(c,\ell,u)$, not merely worst-case optimal.

\subsection{Proof of Proposition~\ref{prop:vc_soundness}}

\begin{proof}[Proof of Proposition~\ref{prop:vc_soundness}]
Fix an arbitrary $x\in\calX$ and a row $i$. By~\eqref{eq:vc_score_bounds}, $s_i(x)\in\calB_i$, so $s_i(x)$ is feasible for the minimization defining $L_i$. Therefore
\begin{equation}
  L_i
  \le
  \sum_{j=1}^K\softmax(s_i(x))_j z_{ij}^L.
  \label{eq:app_li_feasible}
\end{equation}
The entries of $\softmax(s_i(x))$ are nonnegative, and by~\eqref{eq:vc_value_bounds} we have $z_{ij}^L\le z_{ij}(x)$ for every $j$. Multiplying termwise by the softmax weights and summing preserves the inequality:
\begin{equation}
  \sum_{j=1}^K\softmax(s_i(x))_j z_{ij}^L
  \le
  \sum_{j=1}^K\softmax(s_i(x))_j z_{ij}(x).
  \label{eq:app_value_lb}
\end{equation}
Combining~\eqref{eq:app_li_feasible} and~\eqref{eq:app_value_lb}, $L_i$ lower-bounds the signed row contribution at input $x$. Summing over rows then gives
\begin{align}
  L_{\mathrm{VC}}
  &=
  b+\sum_i L_i \nonumber\\
  &\le
  b+\sum_i\sum_{j=1}^K\softmax(s_i(x))_j z_{ij}(x)
  \le m(x),
\end{align}
where the last inequality is exactly the margin interface~\eqref{eq:vc_margin_interface}. Since $x\in\calX$ was arbitrary, $L_{\mathrm{VC}}\le m(x)$ throughout $\calX$. If $L_{\mathrm{VC}}>0$, then $m(x)>0$ for all $x\in\calX$, certifying the margin.
\end{proof}

\section{Vertex-CROWN Construction Details}
\label{app:vc_details}

For the architectures used in our experiments, the input image is first mapped to token embeddings $H(x)\in\R^{R\times d}$. For head $h$, the queries, keys, and values are affine maps
\begin{equation}
\begin{aligned}
  Q_i^h(x)&=W_Q^h H_i(x)+b_Q^h,\\
  K_j^h(x)&=W_K^h H_j(x)+b_K^h,\\
  V_j^h(x)&=W_V^h H_j(x)+b_V^h.
\end{aligned}
\end{equation}
The attention score is
\begin{equation}
  s_{ij}^h(x)=\frac{1}{\sqrt{d_h}}\langle Q_i^h(x),K_j^h(x)\rangle+\mu_{ij}^h,
  \label{eq:attention_score_def}
\end{equation}
where $\mu_{ij}^h$ is a fixed mask or positional score term. The experiments compute $[\ell_{ij}^h,u_{ij}^h]$ by applying auto\_LiRPA/CROWN to~\eqref{eq:attention_score_def}. A simple fallback is interval product bounding: from scalar bounds $q_{ir}^h\in[\underline q_{ir}^h,\overline q_{ir}^h]$ and $k_{jr}^h\in[\underline k_{jr}^h,\overline k_{jr}^h]$, set
\begin{align}
  \underline p_{ijr}^h
  &=\min\{\underline q\,\underline k,\underline q\,\overline k,
        \overline q\,\underline k,\overline q\,\overline k\},\\
  \overline p_{ijr}^h
  &=\max\{\underline q\,\underline k,\underline q\,\overline k,
        \overline q\,\underline k,\overline q\,\overline k\},
\end{align}
where the four products use the corresponding $(i,j,r,h)$ bounds. Then
\begin{equation}
  \ell_{ij}^h=\mu_{ij}^h+\frac{1}{\sqrt{d_h}}\sum_r\underline p_{ijr}^h,
  \qquad
  u_{ij}^h=\mu_{ij}^h+\frac{1}{\sqrt{d_h}}\sum_r\overline p_{ijr}^h
  \label{eq:score_box_ibp}
\end{equation}
is sound. CROWN/McCormick multiplication relaxations can be substituted; Vertex-Softmax only needs the certified scalar intervals.

Let the attention-residual state after the output projection be
\begin{equation}
\begin{aligned}
  H_i^+(x)&=H_i(x)+b_O+\sum_h W_O^h O_i^h(x),\\
  O_i^h(x)&=\sum_j a_{ij}^h(x)V_j^h(x).
\end{aligned}
  \label{eq:attention_residual_state}
\end{equation}
For a target margin $m_t(x)$, suppose the downstream classifier or CROWN suffix gives an affine lower bound
\begin{equation}
  m_t(x)\ge \beta_t+\sum_i \gamma_{ti}^{\top}H_i^+(x)
  \qquad \forall x\in\calX.
  \label{eq:suffix_affine_bound}
\end{equation}
Substituting~\eqref{eq:attention_residual_state} produces attention coefficients
\begin{equation}
  \eta_{tih}=(W_O^h)^{\top}\gamma_{ti},
  \qquad
  c_{tihj}\le \eta_{tih}^{\top}V_j^h(x)\quad \forall x\in\calX.
  \label{eq:value_coeff_def}
\end{equation}
When $V_j^h(x)$ is affine in an input box $x\in[x^L,x^U]$, the lower bound in~\eqref{eq:value_coeff_def} is
\begin{equation}
  \operatorname{LB}_{x\in[x^L,x^U]}[w^\top x+b]
  =(w^+)^\top x^L+(w^-)^\top x^U+b,
  \label{eq:affine_input_box_lb}
\end{equation}
with $w^+=\max(w,0)$ and $w^-=\min(w,0)$. The non-attention pieces are gathered into
\begin{equation}
  b_t'=\beta_t+
  \operatorname{LB}_{x\in\calX}\left[
    \sum_i\gamma_{ti}^{\top}H_i(x)+\sum_i\gamma_{ti}^{\top}b_O
  \right].
  \label{eq:b_prime_def}
\end{equation}
The target-wise Vertex-CROWN lower bound is
\begin{equation}
  L_{\mathrm{VC},t}
  =b_t'+\sum_{h,i}L_{\mathrm{box}}\!\left(c_{tih:},\ell_{ih:},u_{ih:}\right).
  \label{eq:full_target_vc_bound}
\end{equation}

\begin{proposition}[Soundness with affine suffix lower bounds]
\label{prop:suffix_hybrid_soundness}
Fix a target class $t$. Suppose the suffix of an attention block admits a sound affine lower bound~\eqref{eq:suffix_affine_bound}; the attention scores satisfy $\ell_{ij}^h\le s_{ij}^h(x)\le u_{ij}^h$ for all $x\in\calX$; and the value-side coefficients satisfy~\eqref{eq:value_coeff_def}. Then $L_{\mathrm{VC},t}$ in~\eqref{eq:full_target_vc_bound} is a sound lower bound. If $L_{\mathrm{CROWN},t}$ is any other sound lower bound on the same target margin, then $\max\{L_{\mathrm{CROWN},t},L_{\mathrm{VC},t}\}$ is sound as well.
\end{proposition}

\begin{proof}
Start from~\eqref{eq:suffix_affine_bound} and substitute~\eqref{eq:attention_residual_state}. The non-attention affine terms are lower-bounded by $b_t'$ in~\eqref{eq:b_prime_def}. For each head $h$ and row $i$, softmax weights are nonnegative and sum to one, so~\eqref{eq:value_coeff_def} gives
\begin{equation}
  \sum_j a_{ij}^h(x)\eta_{tih}^{\top}V_j^h(x)
  \ge
  \sum_j a_{ij}^h(x)c_{tihj}.
\end{equation}
The score vector $s_i^h(x)$ is feasible for the score box $[\ell_{ih:},u_{ih:}]$, so Theorem~\ref{thm:threshold} yields
\begin{equation}
  \sum_j a_{ij}^h(x)c_{tihj}
  \ge L_{\mathrm{box}}\!\left(c_{tih:},\ell_{ih:},u_{ih:}\right).
\end{equation}
Summing over rows and heads recovers~\eqref{eq:full_target_vc_bound}. The maximum of two lower bounds on the same scalar margin is again a lower bound.
\end{proof}

\begin{algorithm}[t]
\caption{Target-wise CROWN/Vertex hybrid}
\label{alg:target_hybrid}
\begin{algorithmic}[1]
\STATE \textbf{Input:} input box $\calX$, true class $y$, target set $\calT=\{t:t\ne y\}$
\FOR{each target $t\in\calT$}
  \STATE Compute direct CROWN margin lower bound $L_{\mathrm{CROWN},t}$
  \STATE Construct score boxes $(\ell_{ih:},u_{ih:})$ for all rows and heads
  \STATE Construct suffix affine lower bound $(\beta_t,\gamma_t)$ over $H^+$
  \STATE Convert $\gamma_t$ to value coefficients $c_{tihj}$ using~\eqref{eq:value_coeff_def}
  \STATE $L_{\mathrm{VC},t}\leftarrow b_t' + \sum_{h,i} L_{\mathrm{box}}(c_{tih:},\ell_{ih:},u_{ih:})$
  \STATE $L_{\mathrm{hyb},t}\leftarrow \max\{L_{\mathrm{CROWN},t},L_{\mathrm{VC},t}\}$
\ENDFOR
\STATE \textbf{return} certified iff $\min_{t\in\calT}L_{\mathrm{hyb},t}>0$
\end{algorithmic}
\end{algorithm}

\section{Numerical Evaluation and Certificate Soundness}
\label{app:numerics}

The exactness statements are real-arithmetic statements. The optimized implementation used in our experiments evaluates Algorithm~\ref{alg:threshold} with standard floating-point PyTorch operations, including the stability shift $a=\max_j u_j$. This is the usual implementation path for empirically comparing verifier bounds, but it is not by itself a proof-carrying floating-point certificate.

For proof-producing use, the same $K+1$ threshold structure can be evaluated conservatively with outward-rounded interval arithmetic. After the stability shift, compute intervals $\widehat L_j$ and $\widehat U_j$ enclosing $\exp(\ell_j-a)$ and $\exp(u_j-a)$. For each threshold $m$, form interval enclosures for
\begin{align}
  N_m &= \sum_{j\le m}c_{(j)}U_{(j)}+\sum_{j>m}c_{(j)}L_{(j)},\\
  D_m &= \sum_{j\le m}U_{(j)}+\sum_{j>m}L_{(j)}.
\end{align}
Because $D_m>0$, interval division yields an enclosure $[\underline\tau_m,\overline\tau_m]$ of the exact threshold value. Then
\begin{equation}
  \underline L_{\mathrm{box}}=\min_{m=0,\ldots,K}\underline\tau_m
  \le L_{\mathrm{box}}(c,\ell,u)
\end{equation}
is a conservative lower bound for the score-box subproblem. A high-precision reference implementation based on Python \texttt{Decimal} arithmetic with outward padding is provided in the supplementary material. The reported GPU certified rates and lower bounds use the faster PyTorch implementation unless explicitly marked as interval-certified.

\clearpage

\section{Additional Experimental Evidence}
\label{app:additional_experiments}

This appendix collects the supporting evidence behind the main experimental claims. To help navigate the material, here is a brief guide to what each subsection addresses and the key tables within it:
\begin{itemize}
  \item \textbf{Protocol and numerical status} (Tables~\ref{tab:protocol_app}--\ref{tab:numerical_status_app}): Can these results be reproduced, and are the floating-point outputs trustworthy?
  \item \textbf{GaLileo-style baseline} (Table~\ref{tab:galileo_baseline_app}): How does Vertex-Softmax compare to the closest prior method on the same score-box problem?
  \item \textbf{Solver correctness and slack} (Tables~\ref{tab:threshold_runtime_app}--\ref{tab:slack_decomposition_app}): Does the $K{+}1$ threshold solver match brute-force enumeration, and once softmax is solved exactly, where does the remaining looseness come from?
  \item \textbf{Runtime accounting} (Table~\ref{tab:runtime_breakdown_app}): Is the hybrid verifier too expensive relative to direct CROWN?
  \item \textbf{Paired improvements and stability} (Tables~\ref{tab:paired_nonmnist_app}--\ref{tab:bootstrap_ci_app}): Are the gains broad-based across images, and do they hold across training seeds?
  \item \textbf{Strong-baseline details} (Tables~\ref{tab:paired_strong_baseline_app}--\ref{tab:selected_strong_baselines_app}): Full per-trial comparisons with alpha-CROWN and ABCrown-BaB.
\end{itemize}

\subsection{Protocol and numerical status}

A natural concern with any verifier comparison is whether the evaluation is fair and the floating-point outputs are trustworthy. Table~\ref{tab:protocol_app} records the evaluation protocol, denominators, and hardware for each benchmark so that results can be reproduced. Table~\ref{tab:numerical_status_app} addresses numerical soundness: it separates the proof-producing interval-oracle checks for the Vertex-Softmax primitive from the standard floating-point verifier outputs used for GPU-scale experiments.

\begin{table*}[ht]
\centering
\setlength{\tabcolsep}{5pt}
\caption{Experiment protocol summary. All image inputs use clipped pixel $\ell_\infty$ boxes in $[0,1]$. Unless otherwise noted, score boxes are computed via auto\_LiRPA/CROWN on the scalar score module, and all reported certificates are standard floating-point verifier outputs. Hardware is RTX~5090 except where marked ``mixed.'' For learned patch-attention rows and conditional appendix rows, the listed $N$ is the number of clean-correct examples selected for certification per seed. For the Fashion/CIFAR full-block certificates in Table~\ref{tab:full_mlp_nonmnist}, $N$ is the first 100 evaluation examples per seed and clean-incorrect examples are counted as uncertified. Synthetic rows compute rates over all listed trials.}
\label{tab:protocol_app}
\begin{tabular}{@{}L{0.24\textwidth}L{0.32\textwidth}cL{0.10\textwidth}L{0.15\textwidth}@{}}
\toprule
Benchmark & Score / suffix method & Seeds & $N$ / seed & Clean acc. \\
\midrule
\multicolumn{5}{@{}l}{\textit{Synthetic score boxes}} \\[2pt]
Scalable sweep & interval-product; no learned suffix & 3 & 200 trials & n/a \\
$\alpha$/ABCrown blocks & synthetic blocks; CROWN score boxes & 3 & selected & n/a \\
\midrule
\multicolumn{5}{@{}l}{\textit{Learned patch-attention models}} \\[2pt]
MNIST binary, 16 tok.  & CROWN; linear classifier & 3 & 1000 & 0.992 \\
MNIST binary, 49 tok.  & CROWN; linear classifier & 2 &  500 & 0.993 \\
MNIST 10-class, 16 tok. & CROWN; linear classifier & 3 &  500 & 0.875 \\
Residual MHA & CROWN; output proj.\ + linear & 3 & 500 & 0.998 \\
\midrule
\multicolumn{5}{@{}l}{\textit{Full attention--residual--MLP blocks}} \\[2pt]
MNIST block & CROWN; target-wise suffix over $H^+$ & 3--5 & 200 & 0.749 \\
Fashion-MNIST block & CROWN; target-wise suffix over $H^+$ & 5 & 100 eval & 0.708/0.672$^{\dagger}$ \\
CIFAR-10 gray block & CROWN; target-wise suffix over $H^+$ & 5 & 100 eval & 0.346 \\
\bottomrule
\multicolumn{5}{@{}p{0.94\textwidth}@{}}{\rule{0pt}{2.2ex} $^{\dagger}$\,0.708 for $d{=}32,h{=}4$; 0.672 for $d{=}16,h{=}2$. ABCrown-BaB uses a fixed 600s per-instance budget.}\\
\multicolumn{5}{@{}p{0.94\textwidth}@{}}{Runtime claims use RTX~5090 rows only; Fashion/CIFAR rows use mixed GPUs.}
\end{tabular}
\end{table*}

\begin{table*}[ht]
\centering
\setlength{\tabcolsep}{2pt}
\caption{Numerical status checks. The oracle interval rows validate the Vertex-Softmax threshold primitive itself. The Fashion-MNIST rows use the exact evaluation-denominator full-block rerun at $\epsilon=0.02$: FP cert. is counted over 500 evaluation examples, while near-zero counts inspect the clean-correct margins that were actually bounded. These checks do not on their own constitute end-to-end outward-rounded LiRPA certificates.}
\label{tab:numerical_status_app}
\begin{tabular}{L{0.23\textwidth}L{0.28\textwidth}rrL{0.13\textwidth}L{0.13\textwidth}}
\toprule
Check & Scope & $N$ & FP cert. & Near-zero $(0,10^{-6}]$ & Near-zero $(0,10^{-4}]$ \\
\midrule
Decimal interval regression & random score boxes, $K\le8$ & 360 & n/a & 0 violations & 0 violations \\
Decimal interval regression & large dynamic range score box & 1 & n/a & 0 violations & 0 violations \\
Fashion-MNIST $\epsilon=0.02$ & CROWN image margins & 500 & 28 & 0 & 0 \\
Fashion-MNIST $\epsilon=0.02$ & Vertex-CROWN image margins & 500 & 70 & 0 & 0 \\
Fashion-MNIST $\epsilon=0.02$ & Hybrid image margins & 500 & 70 & 0 & 0 \\
\bottomrule
\end{tabular}
\end{table*}

The interval-oracle checks confirm that the Vertex-Softmax primitive produces correct results to within numerical tolerance. For the GPU-scale experiments, no certified margins fall in the near-zero band where floating-point rounding could flip a certification decision.

\subsection{GaLileo-style implementation baseline}

Table~\ref{tab:galileo_baseline_app} reports the GaLileo-style implementation baseline for $K\le16$. Because the original GaLileo code is not publicly available, these results are from our independent reimplementation and are reported only as an implementation baseline, not as an official reproduction. Vertex-Softmax is tighter in every tested case.

\begin{table}[ht]
\centering
\setlength{\tabcolsep}{3pt}
\caption{GaLileo-style implementation baseline at $\epsilon=0.02$, three seeds, 600 trials per row. These results are from our independent reimplementation (original code not publicly available).}
\label{tab:galileo_baseline_app}
\begin{tabular}{rlrrrr}
\toprule
$K$ & Method & Cert. rate & Mean lower & Gap & Time (s) \\
\midrule
4   & GaLileo-style & 0.165 & -0.2164 & 0.4707 & 1.8171 \\
4   & Vertex        & 0.180 & -0.2020 & 0.4563 & 0.0298 \\
8   & GaLileo-style & 0.037 & -0.3623 & 0.5230 & 4.9151 \\
8   & Vertex        & 0.053 & -0.3367 & 0.4973 & 0.0299 \\
16  & GaLileo-style & 0.002 & -0.4604 & 0.5534 & 19.2317 \\
16  & Vertex        & 0.007 & -0.4280 & 0.5211 & 0.0302 \\
\bottomrule
\end{tabular}
\end{table}

\subsection{Exact solver and slack diagnostics}

Two questions arise about the Vertex-Softmax primitive: does the $K+1$ threshold solver actually match exhaustive enumeration, and once it solves the softmax subproblem exactly, how much looseness remains from other sources? Table~\ref{tab:threshold_runtime_app} addresses the first: the threshold solver agrees with exhaustive search to within $5\times10^{-7}$ for all tested $K$, while running orders of magnitude faster at $K\ge16$. Table~\ref{tab:slack_decomposition_app} addresses the second by decomposing the total verification gap on controlled tiny instances into its constituent sources.

\begin{figure}[ht]
  \centering
  \includegraphics[width=0.9\linewidth]{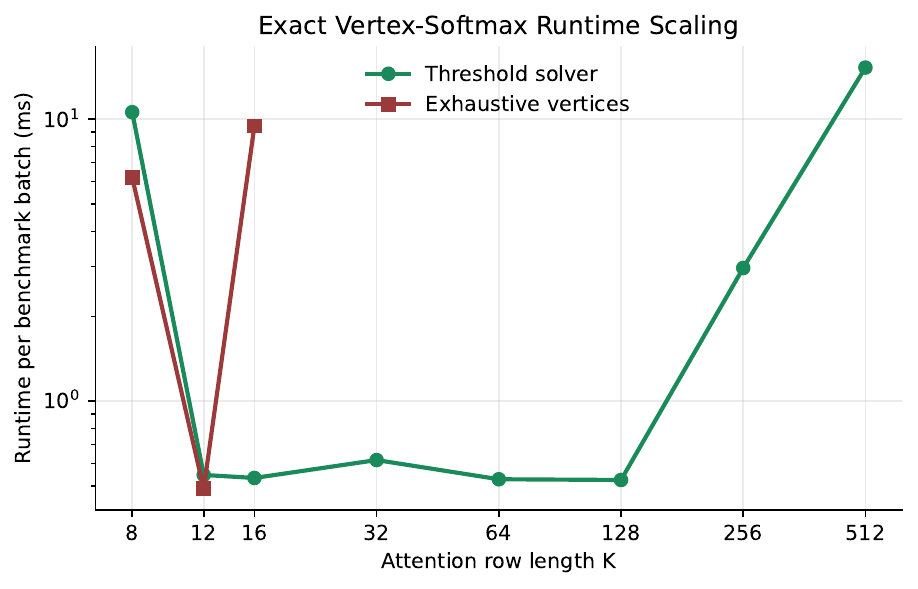}
  \caption{Runtime scaling for the exact Vertex-Softmax threshold solver. Exhaustive enumeration is shown only where feasible.}
  \label{fig:threshold_runtime_app}
\end{figure}

\begin{table}[ht]
\centering
\setlength{\tabcolsep}{4pt}
\caption{Threshold solver exactness and runtime. Exhaustive enumeration is shown only where feasible.}
\label{tab:threshold_runtime_app}
\begin{tabular}{rrrr}
\toprule
$K$ & Threshold (s) & Exhaustive (s) & Max diff. \\
\midrule
8   & 0.00046 & 0.00144 & $3.6{\times}10^{-7}$ \\
12  & 0.00033 & 0.00725 & $2.4{\times}10^{-7}$ \\
16  & 0.00037 & 0.11383 & $4.8{\times}10^{-7}$ \\
64  & 0.00162 & --      & -- \\
128 & 0.00453 & --      & -- \\
\bottomrule
\end{tabular}
\end{table}

\begin{table*}[ht]
\centering
\setlength{\tabcolsep}{4pt}
\caption{Slack-decomposition diagnostics on controlled tiny instances. Entries are mean lower-bound gaps over random seeds; larger values mean more looseness from that source. The softmax gap is Vertex-Softmax minus Wei-LSE on the same score boxes. Grid-based gaps use a dense two-dimensional input grid and are diagnostics rather than proof-producing certificates.}
\label{tab:slack_decomposition_app}
\begin{tabular}{lrrrrrrr}
\toprule
Diagnostic & $\epsilon$ & Softmax & Score box & Value & Row-wise & Suffix & Final \\
\midrule
single-row attention & 0.20 & 0.048 & 0.048 & 0.157 & -- & -- & 0.205 \\
two-row residual & 0.20 & 0.168 & 0.228 & -- & 0.151 & -- & 0.379 \\
ReLU suffix & 0.20 & -- & -- & -- & -- & 0.625 & 0.625 \\
\bottomrule
\end{tabular}
\end{table*}

The slack decomposition confirms that once the softmax subproblem is solved exactly, the dominant remaining sources of looseness are value-bound slack and, when present, suffix relaxation. This supports the claim that Vertex-Softmax removes the most accessible source of avoidable looseness at the score-box interface.

\subsection{Runtime accounting}

A practical concern is whether the Vertex-CROWN hybrid is too expensive relative to direct CROWN. Table~\ref{tab:runtime_breakdown_app} breaks down the wall time for the full-block MNIST setting. The Vertex-Softmax sort/sweep itself is negligible ($0.02\%$ of total time); the cost of the hybrid comes almost entirely from constructing the score, value, and suffix bounds that feed into the oracle.

\begin{table}[ht]
\centering
\setlength{\tabcolsep}{3pt}
\caption{Coarse runtime accounting for the RTX 5090 full-block MNIST setting $d=32,h=4,m=64,\epsilon=0.02$. Direct CROWN, Vertex path, and Hybrid totals are averaged over three seed runs with 200 certified images per seed. The Vertex-Softmax sort/sweep entry is estimated by scaling the measured $K=16$ threshold microbenchmark to $9$ targets, $4$ heads, and $16$ query rows; the remaining Vertex-path time is therefore attributed to score, value, and suffix-bound construction.}
\label{tab:runtime_breakdown_app}
\begin{tabular}{L{0.48\columnwidth}rr}
\toprule
Component & Sec./img & Fraction \\
\midrule
Direct CROWN & 0.1058 & 0.094 \\
Score/value/suffix bounds & 1.0147 & 0.899 \\
Vertex sort/sweep & 0.0002 & 0.000 \\
Hybrid overhead & 0.0081 & 0.007 \\
Total hybrid & 1.1288 & 1.000 \\
\bottomrule
\end{tabular}
\end{table}

\subsection{Paired improvements and seed-resampling stability}

The main-text tables report aggregate certified rates and mean lower bounds. Here we ask two finer-grained questions: does the hybrid improve bounds on most individual images (not just on average), and are the headline improvements stable across training seeds?

Table~\ref{tab:paired_nonmnist_app} reports paired improvements from the same exact evaluation-denominator reruns as Table~\ref{tab:full_mlp_nonmnist}. At $\epsilon=0.02$ on Fashion-MNIST, the hybrid improves the lower bound on $99.7\%$ of clean-correct evaluated images, with a median improvement of $4.26$ points, confirming that the gains are broad-based rather than driven by a few outliers. Table~\ref{tab:bootstrap_ci_app} gives seed-resampling intervals for the main certified-rate improvements. Because several rows have only two or three seeds, these intervals are descriptive stability checks rather than formal significance tests; nonetheless, the resampled intervals are comfortably above zero in every case.

\begin{table*}[ht]
\centering
\setlength{\tabcolsep}{4pt}
\caption{Paired lower-bound improvements for the non-MNIST full-block rows from the exact evaluation-denominator reruns. Mean $\Delta$ is seed-averaged to match Table~\ref{tab:full_mlp_nonmnist}; median $\Delta$ and fraction improved are computed over the clean-correct evaluated images whose margins were bounded.}
\label{tab:paired_nonmnist_app}
\begin{tabular}{llrrrr}
\toprule
Dataset / setting & $\epsilon$ & $N$ & Mean $\Delta$ & Median $\Delta$ & Fraction improved \\
\midrule
Fashion-MNIST $d=32,h=4,m=64$ & 0.01  & 358 &  0.215 &  0.000 & 0.307 \\
Fashion-MNIST $d=32,h=4,m=64$ & 0.02  & 358 &  6.402 &  4.256 & 0.997 \\
Fashion-MNIST $d=32,h=4,m=64$ & 0.03  & 358 & 33.608 & 27.577 & 1.000 \\
Fashion-MNIST $d=16,h=2,m=64$ & 0.02  & 344 &  4.608 &  2.254 & 0.971 \\
CIFAR-10 gray $d=32,h=4,m=64$ & 0.005 & 188 &  0.182 &  0.000 & 0.277 \\
CIFAR-10 gray $d=32,h=4,m=64$ & 0.01  & 188 &  4.521 &  2.680 & 1.000 \\
\bottomrule
\end{tabular}
\end{table*}

\begin{table*}[ht]
\centering
\setlength{\tabcolsep}{3pt}
\caption{Seed-resampling stability checks for selected certified-rate improvements using stored summaries. Rows marked e2e use exact evaluation-denominator rates; the remaining learned image-model rows use conditional rates on clean-correct certification subsets. The intervals are percentile resampling intervals over training/evaluation seeds; rows with two or three seeds should be read as descriptive stability checks rather than formal significance tests.}
\label{tab:bootstrap_ci_app}
\begin{tabular}{L{0.30\textwidth}ccrrrL{0.15\textwidth}}
\toprule
Setting & $\epsilon$ & Seeds & CROWN (\%) & Method (\%) & $\Delta$ pp & Resampled interval pp \\
\midrule
MNIST patch attention, $K=16$ (Vertex) & 0.03 & 3 & 20.3 & 77.8 & +57.5 & [+39.2, +67.4] \\
MNIST patch attention, $K=49$ (Vertex) & 0.03 & 2 & 46.5 & 92.5 & +46.0 & [+32.6, +59.4] \\
MNIST 10-class patch attention (Vertex) & 0.02 & 3 & 3.3 & 12.9 & +9.5 & [+6.0, +12.0] \\
MNIST full block, $d=32,h=4,m=64$ (Hybrid) & 0.02 & 3 & 10.2 & 14.3 & +4.2 & [+2.5, +5.5] \\
Fashion full block, $d=32,h=4,m=64$ (e2e Hybrid) & 0.02 & 5 & 5.6 & 14.0 & +8.4 & [+5.8, +11.2] \\
Fashion full block, $d=16,h=2,m=64$ (e2e Hybrid) & 0.02 & 5 & 9.6 & 25.2 & +15.6 & [+14.2, +16.6] \\
CIFAR-gray full block, $d=32,h=4,m=64$ (e2e Hybrid) & 0.01 & 5 & 0.4 & 2.6 & +2.2 & [+0.8, +3.8] \\
\bottomrule
\end{tabular}
\end{table*}

\subsection{Strong-baseline details}

The main text reports that Vertex-CROWN is competitive with alpha-CROWN and outperforms ABCrown-BaB in selected settings. Tables~\ref{tab:paired_strong_baseline_app} and~\ref{tab:selected_strong_baselines_app} provide the full per-trial details behind those claims.

Table~\ref{tab:paired_strong_baseline_app} shows paired diagnostics: on the two settings where Vertex-CROWN most clearly outperforms alpha-CROWN ($K=4, d=16, \epsilon=0.05$ and $K=4, d=32, \epsilon=0.03$), Vertex-CROWN produces a higher lower bound on $21$--$23$ out of $24$ trials. The ``V-only'' column shows instances certified only by Vertex-CROWN and not by alpha-CROWN; in both settings, Vertex-CROWN certifies several instances that alpha-CROWN misses, while the reverse is rare. Table~\ref{tab:selected_strong_baselines_app} consolidates the full certified rates, lower bounds, and runtimes across all tested settings and methods, including the residual-MHA block where alpha-CROWN's optimizable slopes give it an advantage at $\epsilon=0.03$ but not at $\epsilon=0.05$.

\begin{table*}[ht]
\centering
\setlength{\tabcolsep}{2pt}
\caption{Paired strong-baseline diagnostics from the long-budget small-attention detail CSVs. Alpha-CROWN and Vertex-CROWN lower bounds are compared on the same trials. ABCrown-BaB reports status under a 600 second per-instance budget; unknown, OOM, and internal-error outcomes are counted as uncertified.}
\label{tab:paired_strong_baseline_app}
\begin{tabular}{L{0.15\textwidth}rrrrrrL{0.07\textwidth}L{0.07\textwidth}L{0.09\textwidth}}
\toprule
Setting & $N$ & $\alpha$ cert. & V cert. & ABC safe & Mean $\Delta$ & Med. $\Delta$ & V$>\alpha$ & V-only / $\alpha$-only & ABC unknown / error \\
\midrule
$K=4,d=16,\epsilon=0.05$ & 24 & 12 & 15 & 12 & 0.164 & 0.139 & 21/24 & 4/1 & 10/2 \\
$K=4,d=32,\epsilon=0.03$ & 24 & 6 & 12 & 7 & 0.250 & 0.232 & 23/24 & 6/0 & 15/2 \\
\bottomrule
\end{tabular}
\end{table*}

\begin{table*}[ht]
\centering
\setlength{\tabcolsep}{3pt}
\caption{Selected strong-baseline results used in Figure~\ref{fig:strong_baselines}. ABCrown-BaB uses a 600 second per-instance budget and reports status, so no scalar lower bound is shown for those rows. Residual-MHA rows use the residual multi-head block from the main comparison.}
\label{tab:selected_strong_baselines_app}
\begin{tabular}{L{0.30\textwidth}L{0.20\textwidth}rrr}
\toprule
Setting & Method & Cert. & Lower & Sec./trial \\
\midrule
$K=4,d=16,\epsilon=0.03$ & alpha-CROWN  & 0.742 &  0.2912 &   1.662 \\
                                  & Vertex-CROWN & 0.742 &  0.2809 &   0.042 \\
\midrule
$K=4,d=16,\epsilon=0.05$ & CROWN         & 0.125 & -1.0350 &   0.158 \\
                                  & alpha-CROWN  & 0.500 & -0.0311 &   1.537 \\
                                  & Vertex-CROWN & 0.625 &  0.1327 &   0.033 \\
                                  & ABCrown-BaB  & 0.500 & --      & 267.806 \\
\midrule
$K=4,d=32,\epsilon=0.03$ & CROWN         & 0.083 & -1.2730 &   0.190 \\
                                  & alpha-CROWN  & 0.250 & -0.1853 &   1.752 \\
                                  & Vertex-CROWN & 0.500 &  0.0645 &   0.037 \\
                                  & ABCrown-BaB  & 0.292 & --      & 395.091 \\
\midrule
$K=6,d=16,\epsilon=0.03$ & alpha-CROWN  & 0.583 &  0.1219 &   1.678 \\
                          & Vertex-CROWN & 0.591 &  0.1255 &   0.042 \\
\midrule
Residual MHA, $\epsilon=0.03$    & CROWN        & 0.313 &  -3.901 &   0.177 \\
                                  & Vertex-CROWN & 0.337 &  -2.631 &   0.071 \\
                                  & alpha-CROWN  & 0.410 &  -0.455 &   3.729 \\
\midrule
Residual MHA, $\epsilon=0.05$    & CROWN        & 0.000 & -39.475 &   0.169 \\
                                  & Vertex-CROWN & 0.163 & -13.054 &   0.070 \\
                                  & alpha-CROWN  & 0.067 & -16.334 &   3.672 \\
\bottomrule
\end{tabular}
\end{table*}

\end{document}